\documentclass[preprint,12pt]{elsarticle}

\usepackage{amssymb}
\usepackage{amsmath}
\usepackage{amsthm}
\usepackage{enumitem}
\usepackage[colorlinks,bookmarks]{hyperref}

\newcommand{\tpami}{IEEE Transactions on Pattern Analysis and Machine Intelligence}
\newcommand{\ijcv}{Internation Journal of Computer Vision}
\newcommand{\cviu}{Computer Vision and Image Understanding}

\newcommand{\proc}{Proceedings of}

\def\varsim{\mathop{\sim}\limits}

\renewcommand{\vec}[1]{\mathbf{#1}}
\newcommand{\mset}[1]{\mathfrak{#1}}

\newtheorem{theorem}{Theorem}
\newtheorem{proposition}{Proposition}

\begin{document}

\begin{frontmatter}



\title{System--theoretic approach to image interest point detection}


\author{Vitaly Pimenov\corref{cor}}
\ead{vitaly.pimenov@gmail.com}
\cortext[cor]{Corresponding author.}

\address{Faculty of Applied Mathematics and Control Processes, Saint-Petersburg State University, Universitetskii pr. 35, Peterhof, Saint-Petersburg 198504, Russia}

\begin{abstract}
Interest point detection is a common task in various computer vision applications. Although a big variety of detector are developed so far computational efficiency of interest point based image analysis remains to be the problem. Current paper proposes a system--theoretic approach to interest point detection. Starting from the analysis of interdependency between detector and descriptor it is shown that given a descriptor it is possible to introduce to notion of detector redundancy. Furthermore for each detector it is possible to construct its irredundant and equivalent modification. Modified detector possesses lower computational complexity and is preferable. It is also shown that several known approaches to reduce computational complexity of image registration can be generalized in terms of proposed theory.

\end{abstract}


\end{frontmatter}


\section{Introduction}
\label{introduction}

In many computer vision and multimedia retrieval applications images are represented as sets of distinctive regions called \emph{interest points} or \emph{keypoints}. In order to select such regions image is processed with \emph{detectors} that usually apply specific local operators to image and select pixels of high response values. Due to their local nature keypoints possess attractive properties, such as stability under various image transforms. Compared to low--level global image features, for instance, color features, interest points are more reliable.
Detected points are characterized by \emph{descriptors}, vectors that fulfill several conditions among which especially important ones are \emph{invariance} to desired image transforms and \emph{distinctiveness}.

Detection--description image processing scheme was found quite effective on practice. It has been utilized in a broad range of applications including content--based image and video retrieval, image registration, stereo reconstruction, robotic navigation, medical imaging, object recognition, copyright infringement detection, computational photography and others. Probably the most successful approach to interest point detection and description proposed so far is Lowe's \emph{scale invariant feature transform (SIFT)} \citep{lowe}. Surveys of modern detectors and descriptors can be found in papers \citep{Mikolajczyk05Affine,surf}.

However there remain challenges related to computational efficiency of concerned image processing methods and quality of their results. \emph{First}, detectors produce large amount of keypoints, around 2~000 for usual images \citep{Jegou09}. It makes hard to implement scalable image processing systems, taking into account computational complexity of descriptor calculation and required storage capacity. For example, web--scale image retrieval systems have to handle collections containing billions of images; storing SIFT descriptors (comprised by 128 floating point values) for 1 billion images with 2~000 keypoints in each would require over 1 000 terabytes of physical memory. Developing reliable retrieval in a large--scale collections is evidently a difficult problem too. Another example could be real--time tasks, as robot vision or interactive tomography. In these scenarios processing time is restricted and common interest point based methods are hardly applicable. But nevertheless emergence of a variety of hardware implementations of SIFT proves the demand on such methods in real--time problem domain.

\emph{Second}, evaluation of interest points' quality is an arguable topic. Theory guarantees that points found in a reference image will be redetected if an image would undergo specific transform (usually is it a similarity of affine transform of image geometry and monotonic intensity change ) and descriptors of corresponding points will be identical. However, when it comes to actual images, assumption of transform type is often violated. This is due to many reasons: three dimensional nature of scenes, occlusions, complex motion, multiple and moving light sources, sensor distortions, noise, lossy compression, complicated editing effects and other. Therefore empirical studies are required to assess actual quality of image analysis. Several experimental methodologies could be found in literature \citep{Mikolajczyk05Affine,PimenovGraphicon,ThomsonComer09,zitova}.  However most of them perform passive of a post factum evaluation: result of such experiments are numerical scores that cannot be directly employed to improved method's quality. More recent works \citep{thomee,PimenovGraphicon,ThomsonComer09} suggest active evaluation that can be done during method execution in order to predict quality of analysis: results of active evaluation could be easily used to reject low quality points.

In current research it is shown that although existing active evaluation approaches have considerable differences it is possible to develop a generalized system--theoretic framework for quantitative evaluation of interest point detectors.
Remaining sections are organized as follows. Section~\ref{related-work} provides a theoretical background on interest point detection and description. Proposed framework is described in Sect.~\ref{framework}. Section~\ref{practice} describes practical applications of developed theory. Finally Sect.~\ref{conclusion} concludes the paper and outlined directions of further research.

\section{Interest Point Detection and Description}
\label{related-work}
In context of current paper image is defined as a nonnegative smooth bounded function of two variables: $I(\vec{X})$, where $\vec{X} = (x, y) \in \mset{X}$, bounded and connected set. Let us denote as $\{I\}$ a set of all images depicting same physical object. Considering two images $I(\vec{X})$ and $I'(\vec{X})$ of above set, their respective points $\vec{X}$ and $\vec{X}'$ are called corresponding iff these points project the same point of physical object.
It is evident that since corresponding points are known an approximate transform between images can be computed. These statement motivates usage of interest points for image matching and registration. In case when all points of images are equivalent establishing the correspondence requires exhaustive search that is prohibitive. Therefore it is necessary to introduce interest point selection technique.

Let us define an interest point detector as an operation $\Phi: \{I\} \rightarrow 2^{\mset{X}}$ associating an image $I$ with a
set $\bar{\mset{X}} \in \mset{X}$ fulfilling the following conditions:
\begin{enumerate}[topsep=0pt, partopsep=0pt,itemsep=0pt,parsep=2pt]
\item There exists a finite algorithm that implements operation $\Phi$.
\item For each image $I$ set $\bar{\mset{X}} = \Phi(I)$ is finite.
\item For each pair of image $I_1$ and $I_2$ sets $\bar{\mset{X}}_1$ and
$\bar{\mset{X}}_2$ consist of corresponding points:
$\forall\,\vec{X}_1 \in \bar{\mset{X}}_1\,\exists\,!\,\vec{X}_2 \in
\bar{\mset{X}}_2$  and $\forall\,\vec{X}_2 \in \bar{\mset{X}}_2\,
\exists\,!\,\vec{X}_1 \in \bar{\mset{X}}_1$ such that $\vec{X}_2 =
\vec{F}(\vec{X}_1),$
where $\vec{F}(\vec{X})$ is a transform between $I_1$ and $I_2$.
\end{enumerate}

Points $\vec{X} \in \bar{\mset{X}}$ are called interest points. On practice interest points are extremum points of some differential operators on $I$ function.

Since sets of interest points $\bar{\mset{X}}_1$ and
$\bar{\mset{X}}_2$ are computed for images  $I_1$ and $I_2$ the correspondences should be established. For each interest point
$\vec{X}_1 \in \mset{\bar{X}}_1$ we have to find point $\vec{X}_2 \in \mset{\bar{X}}_2 $ such that $\vec{X}_2 = \vec{F}(\vec{X}_1)$.
Unique existence of point $\vec{X}_2$ is guaranteed by definition of detector. It is sufficient to do an exhaustive search over a finite set $\bar{\mset{X}}_2$ to select $\vec{X}_2$. Question of whether points $(\vec{X}_1, \vec{X}_2)$
correspond or not is answered on the basis of \emph{interest point descriptors} \citep{lowe,Schmidt00,Mikolajczyk05Affine}.

Let us consider a metric space $D$ and denote respective metric as $\rho_D$. Operation $\Psi: \{I\} \times \mset{X}
\rightarrow D$ is called an interest point descriptor if following conditions are satisfied:
\begin{enumerate}[topsep=0pt, partopsep=0pt,itemsep=0pt,parsep=2pt]
\item There exists a finite algorithm implementing operation $\Psi$.
\item There exists an $\epsilon_D \geq 0$ such that for each images $I_1$ and $I_2$ and each pair of
corresponding points $\vec{X}_1 \in \bar{\mset{X}}_1$ and
$\vec{X}_2 \in \bar{\mset{X}}_2$  following relationship holds with necessity:
\begin{equation}\label{eq:descriptor-condition}
\rho_D(\Psi(I_1, \vec{X}_1), \Psi(I_2, \vec{X}_2)) \leq \epsilon_D,
\end{equation}
and its violation precludes correspondence between points $\vec{X}_1$ and $\vec{X}_2$.
\end{enumerate}
Value $\Psi(I, \vec{X}) \in D$ is called a description of point $\vec{X}$ of an image $I$.
Interest points such that relationship \eqref{eq:descriptor-condition} holds are called
corresponding in terms of descriptor $\Psi$. Metric space $D$ is called a description space.

Theoretically well--founded way to implement interest descriptor is usage of truncated Taylor series (N--jet)  or directional filter banks \citep{Lindeberg09}.

Formal definitions presented in current section are usually presumed in the scope of interest point based image analysis. However it is apparent that numerical implementations of detectors and descriptors violate strict formal conditions. The major cause of violation lies in discrete nature of images and computation. On practice sets of interest points are redundant. It means that correspondence can be established only between small subsets of interest points. In following section a theory of irredundant interest point detection is developed.

\section{System--theoretic approach to image interest point detection}\label{framework}
Traditionally detection and description of interest points are considered as an isolated and independent stages of image processing. Therefore it is impossible to conclude about redundancy of points during detection stage. It is because the fact of redundancy can appear only after description and matching is performed. Since no knowledge about descriptor is available to detector redundancy cannot be evaluated. In current work it is proposed to utilize system--theoretic approach to unveil the interdependence between detection and description. Using the knowledge about interdependency it is possible to construct irredundant detectors.

Let us begin with introducing a definition of approximate interest point detector called a \emph{$\lambda$--correct detector}.
Given a $\lambda{\geq}0$ detector $\Phi: \{I\} \rightarrow 2^{\mset{X}}$,
is called a $\lambda$--correct detector is following condition is satisfied:
for each transform $\vec{F}(\vec{X})$ and each point $\vec{X}' \in \Phi\left(I(\vec{X})\right)$ there exist an interest point
$\vec{X}'' \in \Phi\left(I\left(\vec{F}(\vec{X})\right)\right)$ such that inequality $\|\vec{F}(\vec{X}') - \vec{X}''\| \leq \lambda$ holds.
Interest points $\vec{X} \in \Phi(I(\vec{X}))$ are called $\lambda$--correct interest points\footnote{Interest point repeatability is a related concept used in literature \citep{Mikolajczyk05Affine,surf}.}.
Value of $\lambda$ is called a correctness level. To denote a $\lambda$--correct detector symbol $\Phi_\lambda$ is employed.
Notion of $\lambda$--correctness allows mathematically rigorous expression of interdependency between interest point detectors and descriptors. On the basis o system--theoretic approach it is possible to build qualitatively new interest point detection theory.

Let us call interest point descriptor $\Psi$ a continuous descriptor if
$\forall\,\vec{F}$ and $\forall\,\epsilon >
0\;\exists\,\delta > 0$ such that
\begin{equation}\label{eq:continuous-descriptor}
\forall\,\vec{X}', \vec{X}'', \|\vec{F}(\vec{X}') -
\vec{X}''\| < \delta\,:\,\rho_D(\Psi(I, X'), \Psi\left(I \circ \vec{F}, X'')\right) <
\epsilon.
\end{equation}
Descriptors build upon image function derivatives possess the above property by virtue of function $I$ smoothness and presumed continuity of transform $\vec{F}$.

Resorting to definition of descriptor, consider $\epsilon = \epsilon_D$. Then there exists
$\delta_D > 0$ such that condition \eqref{eq:continuous-descriptor} holds. Consider now $\Phi_{\delta_D}$ ---
$\delta_D$--correct interest point detector. By definition of $\Psi$ it follows that
for all $\lambda$--correct points detected with $\Phi_{\delta_D}$ corresponding points will be found. Hereinafter such
detector will be referred to as $\Phi_{\Psi}$.

Consider images $I'$, $I''$, transform $\vec{F}$ between them and descriptor $\Psi$. A set of
$\Psi$--irredundant correspondences for interest point detector $\Phi_\lambda$ is a set
$$
\mset{K}_{\Phi_\lambda}(I', I'') = \{(\vec{X}',
\vec{X}'')\,|\,\vec{X}'\in\Phi(I'),\,\vec{X}''\in\Phi(I''),\,\|\vec{F}(\vec{X}')
- \vec{X}''\| < \delta_D\}.
$$
Set $\mset{K}_{\Phi_\lambda}$ consists of interest point pairs $(\vec{X}',
\vec{X}'')$ that correspond in terms of descriptor $\Psi$.

On the basis of $\Psi$--irredundancy it is possible to define equivalence relation between
interest point detectors. Detectors $\Phi_{\lambda_1}$ and $\Phi_{\lambda_2}$
are called $\Psi$--equivalent: $\Phi_{\lambda_2} \varsim_{\Psi}
\Phi_{\lambda_2}$, if for any images $I'$, $I''$ following equality holds
\begin{equation}\label{eq:equivalence}
\mset{K}_{\Phi_{\lambda_1}}(I', I'') = \mset{K}_{\Phi_{\lambda_2}}(I', I'').
\end{equation}

\noindent $\Psi$--equivalence of two detectors means that for any points that is not detected by both detectors there is no corresponding point in terms of descriptor $\Psi$.

\begin{proposition}\label{st:equivalence}
$\Psi$--equivalence relation is an equivalence relation.
\end{proposition}
\begin{proof}
Reflectivity, symmetry and transitivity properties are inherited from set equality relation.
\end{proof}

\noindent Concept of  $\Psi$--equivalence allows to consider detector equivalence classes. For any two detectors belonging to the same class interest point matching results obtained using descriptor $\Psi$ will apparently coincide.

Given $\lambda_1$--correct detector $\Phi_{\lambda_1}$ and $\lambda_2$--correct detector $\Phi_{\lambda_2}$, $\Phi_{\lambda_1}$ is called \emph{embedded} in $\Phi_{\lambda_2}$ $(\Phi_{\lambda_2}$ contains
$\Phi_{\lambda_1})$ and denoted as $\Phi_{\lambda_1} \subseteq
\Phi_{\lambda_2}$, if $\lambda_1 \leq \lambda_2$ and for any image $I \in \{I\}$ following relationship holds:
$\Phi_{\lambda_1}(I) \subseteq \Phi_{\lambda_2}(I),$
for each point $\vec{X} \in \Phi_{\lambda_2}(I)\setminus\Phi_{\lambda_1}(I)$
and each point $\vec{X}' \in \Phi_{\lambda_2}(I \circ \vec{F})$ following inequality holds:
$\|\vec{F}(\vec{X}) - \vec{X}'\| > \lambda_1.$

\begin{theorem}\label{st:existence}
For any $\lambda$--correct interest point detector $\Phi_\lambda$ and any continuous interest point descriptor
$\Psi$ there exists $\Phi_\Psi \subseteq \Phi_\lambda$ --- $\delta_D$--correct interest point detector such that
$\Phi_\Psi \varsim_\Psi \Phi_\lambda$.
\end{theorem}
\begin{proof}
Resorting to definition of continuous descriptor, consider value of $\delta_D$. Two alternatives are available for detector  $\Phi_\lambda$
\begin{enumerate}[topsep=0pt,partopsep=0pt,itemsep=0pt,parsep=2pt]
  \item $\lambda \leq \delta_D$: in this case $\Phi_\lambda$ apparently is a
  $\delta_D$--correct detector and therefore $\Phi_\Psi$ is $\Phi_\lambda$.
  $\Phi_\Psi \varsim_\Psi \Phi_\lambda$ because of reflectivity of equivalence relation.
 \item $\lambda > \delta_D$: let us describe a way to build $\Phi_\Psi$. Consider a fixed continuous transform $\vec{F}$ and an image
 $I$. Then there exists at least one interest point $\vec{X}' \in
 \Phi_\lambda(I)$, such that for each $\vec{X}'' \in
 \Phi_\lambda(I\circ\vec{F})$
 $$
 \|\vec{F}(\vec{X}') - \vec{X}''\| > \delta_D.
 $$
 Let us denote as $\Phi_{\lambda_1}$ interest point detector, such that
 $\Phi_{\lambda_1}(I) = \Phi_{\lambda}(I)\setminus{\vec{X}'}$. This detector is evidently embedded in $\Phi_{\lambda}$. It can be shown that  $\Phi_{\lambda} \varsim_\Psi \Phi_{\lambda_1}$. By definition of
 $\Phi_{\lambda_1}$ it is necessary to prove equivalency conditions only for images $I$ and $I\circ\vec{F}$.
Let us introduce a quantity
 \begin{equation}\label{eq:detection-error}
 e_{I,\vec{F}}(\vec{X}') = \min\limits_{\vec{X}'' \in \Phi_\lambda(I
 \circ\vec{F})}\|\vec{F}(\vec{X'})- \vec{X}''\|,
 \end{equation}
called detection error of $\Phi_\lambda$ at point $\vec{X}'$\footnote{Concept of localization accuracy \citep{zitova,mikolajczyk05} is related to detection error.}. By definition, for point $\vec{X}'$ inequality $e_{I,\vec{F}}(\vec{X}') > \delta_D$ holds. Then
 $\vec{X}'$ is a redundant interest point and cannot belong to any of pairs comprising a set $\mset{K}_{\Phi_\lambda}(I,
 I\circ\vec{F})$. Hence equivalence condition \eqref{eq:equivalence} follows.
\end{enumerate}
Two described above alternatives are available for detector $\Phi_{\lambda_1}$ too. Sequentially carrying out similar analysis we construct a sequence of detectors $\{\Phi_{\lambda_0},
 \Phi_{\lambda_1}, \Phi_{\lambda_2}, \ldots\}$, where $\lambda_0 = \lambda$. It can be shown that there exists an element of sequence $\Phi_{\lambda_{i^*}}$, such that $\lambda_{i^*} \leq \delta_D$.
A set $\Phi_{\lambda}(I)$ is finite by definition, therefore above sequence also has a finite number of elements.
Consider a singular case, when for each $\vec{X}' \in  \Phi_\lambda(I)$ holds $e_{I,\vec{F}}(\vec{X}') > \delta_D$. Then in process of sequence construction all points will be excluded one by one: $\lambda_{i^*} = 0 <
 \delta_D$, corresponding to a trivial 0--correct detector.
Otherwise there exists be at least one interest point $\vec{X}^* \in
 \Phi_\lambda(I)$, such that $e_{I,\vec{F}}(\vec{X}^*) \leq \delta_D$. Then $\lambda_{i^*} = e_{I,\vec{F}}(\vec{X}^*) \leq \delta_D$. In general, there can exists a set of points $\{\vec{X}_1^*,\ldots,\vec{X}_n^*\}$, such that
 $e_{I,\vec{F}}(\vec{X}_i^*) \leq \delta_D$. In this case it follows that
 $$\lambda_{i^*} = \max\limits_{i=\overline{1,n}}e_{I,\vec{F}}(\vec{X}_i^*) \leq
 \delta_D.$$
Hereby a $\delta_D$--correct interest point detector $\Phi_\Psi = \Phi_{\lambda_{i^*}}$ exists and can be constructed with described procedure.
At the same time $\Phi_{\lambda_{i^*}}$ is a last element of detector sequence under analysis. That is because first alternative holds for $\Phi_{\lambda_{i^*}}$.
Statement $\Phi_\Psi \varsim_\Psi \Phi_\lambda$ appears as a result of transitivity of $\Psi$--equivalence and proof and
 $\Psi$--equivalence between any contiguous elements of detector sequence:
 $\Phi_{\lambda_i} \varsim_\Psi \Phi_{\lambda_{i + 1}},\,i=\overline{0,i^*-1}$.
\end{proof}

By means of $\lambda$--correct detector theory it is possible to compare existing detectors measuring their correctness level.
However it is more important to evaluate redundancy of detector in the scope of given applications. Such evaluation can be carried out within developed framework on the basis of system--theoretic approach. We have already shown an interdependency between $\lambda$--correct interest point detectors and descriptors. To evaluate redundancy it is necessary to dispose of  $\delta_D$ value (cf. definition of descriptor continuity). This value can be estimated given $\epsilon_D$ value, that defines interest point description distinction threshold. This value can be by--turn evaluated on the assumption of quality measures specific to an area of application. For example, transform approximation error can be a quality measure for image registration: given the value of admissible error $\epsilon_D$ value can be estimated experimentally. Systematic approach to interdependencies between stages of image processing allows therefore to evaluate redundancy of interest point detectors.

\section{Practical Applications}\label{practice}
Among possible practical applications of $\lambda$--correct detector theory reducing computational complexity of image matching is in paper's focus.

\begin{theorem}\label{st:complexity}
Consider a set of  $\Psi$--equivalent interest point detectors
$\{\Phi_{\lambda_k}\}$. Let among elements of this set exist $\Phi_{\lambda_k^*}$ such that $\lambda_k^* > \delta_D$ and
$\Phi_{\lambda_{*k}}$ such that $\lambda_{*k} \leq \delta_D$. Then an ordering relationship $\prec$ can be established upon the set $\{\Phi_{\lambda_k}\}$: $\Phi_{\lambda_1} \prec
\Phi_{\lambda_2} \Longleftrightarrow \lambda_1 \leq \delta_D,\, \delta_D <
\lambda_2$. And for all $\Phi_{\lambda_{*k}} \prec \Phi_{\lambda_k^*}$
hold the inequality
$$
c(\Phi_{\lambda_{*k}}) < c(\Phi_{\lambda_k^*}),
$$
where $c(\Phi_\lambda)$ is a number of $\rho_D$ value calculations required to establish correspondences between interest points detected by means of $\Phi_\lambda$.
\end{theorem}
\begin{proof}
Inequality $\lambda_k^* > \delta_D$ means that among points detected with $\Phi_{\lambda_k^*}$ there will inevitably be $\lambda_k^*$--correct interest points that are not $\delta_D$--correct. Let $n^*$ be a number of such points and $n$ will be a total number of detected points.
By definition of $\Psi$ each of above described $n^*$ is redundant. Therefore $\rho_D$ value calculation for such points is redundant too. Establishing correspondences requires $n (n + 1) / 2$ calculations of metric value:
$$
c(\Phi_{\lambda_k^*}) = \frac{n (n + 1)}{2}.
$$
In case when correspondences are to be established only between $\delta_D$--correct points it follows that
$$
c(\Phi_{\lambda_{*k}}) = \frac{(n - n^*) (n - n^* + 1)}{2}.
$$
\noindent Existence of $n^*$ redundant points results in redundant metric value calculations that have no effect on matching:
$$
c(\Phi_{\lambda_k^*}) - c(\Phi_{\lambda_{*k}}) = \frac{n^*(1 + 2n - n^*)}{2} > 0.
$$
\end{proof}

Theorem~\ref{st:complexity} states that the concept of redundant complexity is defined upon
$\Psi$--equivalent interest point detectors. Theorem~\ref{st:existence}
states that for each $\lambda$--correct interest point detector $\Phi_\lambda$ it is possible to construct $\Psi$--equivalent
irredundant detector $\Phi_\Psi$.
Property of $\Psi$--equivalence guarantees that sets of corresponding interest points computed using descriptor $\Psi$
together with detectors $\Phi_\lambda$ and $\Phi_\Psi$ will coincide. Therefore transform approximations will coincide too. Furthermore, if sets of corresponding points are employed in the application scope to solve problems other than transform estimation, the solutions obtained with $\Phi_\lambda$ and $\Phi_\Psi$ will coincide. The choice between $\Phi_\lambda$
and $\Phi_\Psi$ can thus be based on their computational complexity: by
Theorem~\ref{st:complexity}, $\Phi_\Psi$ is advantageous.

Consider now a question of computing $\Phi_\Psi$ given a $\Phi_\lambda$. Theorem~\ref{st:existence} proof is constructive, but it can't be directly employed in numerical methods, because in the course of proof detection error function $e_{I, \vec{F}}(\vec{X})$ plays a significant role. To compute values of this function we have to know transform
$\vec{F}$ (cf. equation \eqref{eq:detection-error}): this requirement prohibits usage of detection error functions is numerical methods since transform is unknown.
Building detector $\Phi_\Psi$ requires means of indirect estimation of  $\Phi_\lambda$ detection error. It should be noticed that in the course of Theorem~\ref{st:existence} proof detection error function is used only to test the inequality $e_{I, \vec{F}}(\vec{X}) > \delta_D$.
Thus given some function $\hat{e}_{I}(\vec{X})$ such that
$e_{I, \vec{F}}(\vec{X}) > \delta_D \Longleftrightarrow
\hat{e}_{I}(\vec{X}) > \hat{\delta}$, where $\hat{\delta}$ is a constant,
$\hat{e}_{I}(\vec{X})$ can replace $e_{I, \vec{F}}(\vec{X})$
without loss of proof validity.

Indirect estimation of $\Phi_\lambda$ detection error can be obtained with different approaches. The straightforward way lies
in averaging values $e_{I, \vec{F}}(\vec{X})$ precalculated for a sampled images, such that transform $\vec{F}$ is known in advance. Similar procedures were proposed in \citep{Mikolajczyk05Affine} for comparative analysis of several known detectors. However using this approach during image registration is ineffective since it requires a multitude of image transformations.

An alternative lies in machine learning employment. Possibility of such approach is reasoned by the fact that testing inequality $\hat{e}_{I}(\vec{X}) > \hat{\delta}$ can be seen as a binary classification problem \citep{Mitchell97}. Positive class corresponds to $\delta_D$--correct interest points. Within such framework interest point description are to be classified, and the descriptions can be calculated by means of descriptor $\hat{\Psi}$ that can differ from $\Psi$ (descriptor used for matching). An example of such approach is a methodic described in \citep{ThomsonComer09}.

Function $\hat{e}_{I}(\vec{X})$ and $\hat{\delta}$ constant can also be defined explicitly. Article \citep{foo} proposes to employ
Laplacian values to evaluate interest point quality: $\hat{e}_{I}(\vec{X}) = {\Delta}I(\vec{X})$. Value of $\hat{\delta}$
is estimated empirically. However there were is no knowledge about correlation between error function and Laplacian.

Finally, visual attention models can be utilized to estimate detection error. Several studies were carried out to evaluate repeatability of salient interest points \citep{PimenovGraphicon,ViAtDraper05,frintrop2}. Draper and Lionelle \citep{ViAtDraper05} compared NVT model (Neuromorphic Vision Toolkit) \citep{ViAtItti98} with SAFE (Selective Attention as a Front End) model and concluded that using SAFE allows to select interest points with average repeatability over 90\% under similarity transforms.
Model VOCUS (Visual Attention System for Object Detection and Goal-directed Search) \citep{frintrop} also build upon NVT was evaluated in articles \citep{frintrop2,PimenovGraphicon}. Results show that ratio of redundant interest points among salient ones is also lower than 10\%. One disadvantage of using visual attention model is inherent restriction to process only natural images.

To conclude, there are several successful approaches to implement image matching complexity reduction and these approaches can be generalized in terms of building irredundant interest point detector $\Phi_\Psi$. It should be noticed that with help of proposed theory it is possible to carry out evaluative studies to compare described approaches.

\section{Conclusion}\label{conclusion}
In this paper a novel interest point detection theory is proposed. By means of system--theoretic analysis of interdependency between interest point detector and descriptor developed an approach to introduce equivalency relation between detectors that are used together with a fixed descriptor. Formal definition of interest point redundancy allows to prove existence of irredundant detector that can be constructed on the basis of any given detector. It is shown how a theory developed can be employed to reduce computational complexity of interest point matching.

Current approach is centered around the notion of $\lambda$--correct interest point detector that generalizes known concept of detector. Since concept of descriptor is left unchanged, further research will be directed to generalizing the notion of interest point descriptor and applying system--theoretic approach to the problem of image interest point based image analysis at whole.







\end{document}